\documentclass[letterpaper, 12 pt, conference]{ieeeconf}  % Comment this line out if you need a4paper

\IEEEoverridecommandlockouts                              % This command is only needed if 
\let\labelindent\relax

\usepackage{graphics} % for pdf, bitmapped graphics files
\usepackage{epsfig} % for postscript graphics files
\usepackage{mathptmx} % assumes new font selection scheme installed
\usepackage{times} % assumes new font selection scheme installed
\usepackage{amsmath} % assumes amsmath package installed
\usepackage{amssymb}  % assumes amsmath package installed
\usepackage{amsthm}
\usepackage{xcolor}
\usepackage{enumitem}
\usepackage{url}
\allowdisplaybreaks

\newtheoremstyle{myStyle}
  {1ex plus 0.2ex minus 0.1ex} % Space above
  {1ex plus 0.2ex minus 0.1ex} % Space below
  {\itshape}    % Body font
  {}            % Indent amount 
  {\bfseries}   % Thm head font
  {}           % Punctuation after thm head
  {0.5em}       % Space after thm head: \newline = linebreak
  {\thmname{#1}~\thmnumber{#2.}~\thmnote{(#3)}}
\makeatother
\theoremstyle{myStyle} 
\newtheorem{theorem}{Theorem}
\newtheorem{lemma}{Lemma}
\newtheoremstyle{myStyle2}
  {1ex plus 0.2ex minus 0.1ex} % Space above
  {1ex plus 0.2ex minus 0.1ex} % Space below
  {\itshape}    % Body font
  {}            % Indent amount 
  {\bfseries}   % Thm head font
  {}           % Punctuation after thm head
  {0.5em}       % Space after thm head: \newline = linebreak
  {\thmname{#1}~\thmnumber{#2.}~\thmnote{(#3)}}
\makeatother
\theoremstyle{myStyle2} 
\newtheorem{definition}{Definition}
\newtheorem{proposition}{Proposition}

\newtheoremstyle{myStyle3}
  {1ex plus 0.2ex minus 0.1ex} % Space above
  {1ex plus 0.2ex minus 0.1ex} % Space below
  {}    % Body font
  {}            % Indent amount 
  {\bfseries}   % Thm head font
  {}           % Punctuation after thm head
  {0.5em}       % Space after thm head: \newline = linebreak
  {\thmname{#1}~\thmnumber{#2.}~\thmnote{(#3)}}
\makeatother
\theoremstyle{myStyle3} 

\theoremstyle{myStyle3} 
\newtheorem{problem}{Problem}
\title{\LARGE \bf
A Mixed-Integer Approach for Motion Planning of Nonholonomic Robots under Visible Light Communication Constraints
}
\author{Angelo Caregnato-Neto$^{1}$, Marcos R. O. A. Maximo$^{3}$, and Rubens J. M. Afonso$^{2}$% <-this % stops a space
\thanks{*This study was financed in part by the Coordena\c c\~ao de Aperfei\c coamento de Pessoal de N\' ivel Superior - Brasil (CAPES) - Finance Code 001.}% <-this % stops a space
\thanks{$^{1}$Electronic Engineering Division, Instituto Tecnol\'ogico de Aeron\' autica, S\~ ao Jos\' e dos Campos, Brazil. Delft University of Technology, The Netherlands. Corresponding author at \texttt{caregnato.neto@ieee.org} }%
\thanks{$^{2}$ Electronic Engineering Division, Instituto Tecnol\'ogico de Aeron\' autica, S\~ ao Jos\' e dos Campos, Brazil. }%
\thanks{$^{3}$Marcos R. O. A. Maximo is with the Autonomous Computational Systems Lab (LAB-SCA), Computer Science
Division, Instituto Tecnol\' ogico de
Aeron\' autica, S\~ ao Jos\' e dos Campos, Brazil.}%
}
\usepackage[noadjust]{cite}

\begin{document}

\maketitle
\thispagestyle{empty}
\pagestyle{empty}
%%%%%%%%%%%%%%%%%%%%%%%%%%%%%%%%%%%%%%%%%%%%%%%%%%%%%%%%%%%%%%%%%%%%%%%%%%%%%%%%
\begin{abstract}

This work addresses the problem of motion planning for a group of nonholonomic robots under Visible Light Communication (VLC) connectivity requirements. 
In particular, we consider an inspection task performed by a Robot Chain Control System (RCCS), where a leader must visit relevant regions of an environment while the remaining robots operate as relays, maintaining the connectivity between the leader and a base station. 
We leverage Mixed-Integer Linear Programming (MILP) to design a trajectory planner that can coordinate the RCCS, minimizing time and control effort while also handling the issues of directed Line-Of-Sight (LOS), connectivity over directed networks, and the nonlinearity of the robots' dynamics.
The efficacy of the proposal is demonstrated with realistic simulations in the Gazebo environment using the Turtlebot3 robot platform.

\end{abstract}

%%%%%%%%%%%%%%%%%%%%%%%%%%%%%%%%%%%%%%%%%%%%%%%%%%%%%%%%%%%%%%%%%%%%%%%%%%%%%%%%
\section{INTRODUCTION}

Visible Light Communication (VLC) is a wireless transmission technology in which information is exchanged through the modulation of visible light beams \cite{VLC_survey}. 
Theorized for many purposes, from positioning systems \cite{VLC_positioning} to disaster monitoring \cite{VLC_disaster}, a particularly exciting application of VLC is point-to-point communication for autonomous cars \cite{VLC_cars} and robots \cite{VLC_application2}.
Unlike conventional communication methods such as radio frequency, VLC is insensible to   electromagnetic interference, inherently secure, and generates less disturbance in high-density networks \cite{VLC_survey,VLC_application2,VLC_application1}. 
This technology is not strictly dependent on Line-Of-Sight (LOS) between transmitters and receivers, being able to achieve non-LOS communication through surface reflections. However, direct light communication has been demonstrated to provide a tenfold enhancement in power transmission when compared to its non-LOS counterpart, allowing for a substantially better quality of service \cite{VLC_survey,LOS_vs_NLOS}.

In \cite{VLC_application1}, a Robot Chain Control System (RCCS) comprised of nonholonomic robots and supported by VLC communication was demonstrated. An RCCS is a multi-robot platform comprised of a base, transmission relays, and a leader. Its application alongside VLC has been proposed for the inspection and maintenance of pipelines and underground facilities \cite{VLC_application2,VLC_application1,VLC_application3}.   In such situations, the base serves as a central commanding unit, whereas the relays connect the leader, responsible for the task, with the base. This chained communication network allows for the extension of the operational area of the leader. The LOS requirement has been addressed in environments with corners and obstacles with point-to-point connections generated by the relays of the RCCS \cite{VLC_application1}. However, the robots were manually operated, making their coordination challenging and susceptible to human error. Similar concepts of chained network systems have been studied in the context of several multi-robot problems, such as sweep coverage \cite{chained_directed_network} and platooning \cite{platoon}, illustrating the usefulness of this architecture.

The problem of autonomous coordination of multi-robot systems with connectivity maintenance is relevant in applications such as formation control \cite{formation_ctrl}, coverage \cite{coverage}, and patrolling \cite{patrol}. In   \cite{UAV_chain}  the deployment of Unmanned Aerial Vehicles (UAVs) for emergency networks in post-disaster scenarios is studied. A chained network is generated to transmit information back to a relief center. The problem of autonomous exploration of mines with connectivity maintenance between robots and a base station is addressed in \cite{mineRadio}. A global planner computes paths for robots and the non-LOS environment is handled by dropping nodes that incrementally build a network. In \cite{chain_conf}, a modular planning algorithm is designed to generate and maintain a chain between a base station and objectives in dynamic and unpredictable environments. A modified Rapid-exploring Random Tree (RRT) approach was proposed in \cite{RRT_LOS} to handle the decentralized path planning of multi-agent systems (MAS) under omnidirectional LOS communication constraints.
A two-step optimization strategy is proposed in \cite{unconstrainedLOS}  to solve a robot deployment problem with omnidirectional LOS connectivity requirements. In the first step, a pruning algorithm provides an initial solution over a discretized map of the environment. Then, the first solution is refined using unconstrained nonlinear optimization over a continuous space.

%The motion planning of UAVs in a surveillance mission is considered in \cite{grotl}. The robots are expected to survey and transmit information back to a base station. The proposed approach optimizes both paths and network topology jointly.

Alternatively, Mixed-Integer Programming (MIP)  offers a flexible optimization framework that has been successfully demonstrated for motion planning \cite{surveyMIP} of several platforms such as UAVs \cite{Bellingham2003}, differential robots \cite{AR23}, and autonomous cars \cite{MIQP_cars}.
MIP-based optimization models leverage the use of integer variables to jointly handle the key issues of collision avoidance \cite{Schouwenaars2001}, LOS-connectivity maintenance \cite{EJC,LOS_how}, and decision-making \cite{Bellingham2003,Afonso2020}, in the context of MAS motion planning.
%This technique can jointly address the key issues of LOS-connectivity \cite{EJC,LOS_how}, collision avoidance \cite{Schouwenaars2001}, and task allocation \cite{Bellingham2003,Afonso2020}.
 In particular, Mixed-Integer Linear Programming (MILP) can address such complex problems in finite run-times and the use of branch-and-bound algorithms can significantly improve solver performance \cite{morrison2016}. Furthermore, the quality of feasible solutions w.r.t. the global optimum can also be assessed through the evaluation of optimality gaps \cite{gurobi}.

%Alternatively, Mixed-Integer Programming (MIP) offers an optimization framework that has been successful in handling decision and motion planning problems for multi-agent systems  \cite{Bellingham2003, Afonso2020, EJC,surveyMIP}.  
%This technique can jointly address the key issues of LOS-connectivity \cite{EJC,LOS_how}, collision avoidance \cite{Schouwenaars2001}, and task allocation \cite{Bellingham2003,Afonso2020}. In particular, Mixed-Integer Linear Programming (MILP) can address such complex problems in finite run-times and the use of branch-and-bound algorithms can significantly improve solver performance \cite{morrison2016}. Furthermore, the quality of feasible solutions w.r.t. the global optima can also be assessed through the evaluation of optimality gaps \cite{gurobi}.

In this work, we propose a MILP-based motion planning algorithm that allows the autonomous operation of the RCCS in an exploration task under directed LOS-connectivity constraints. The group operates preserving the connections between the base and a leader using robot relays and considering a VLC communication network. Our approach provides three main contributions:
\begin{itemize}[]
	\item Novel MIP constraints for directed-LOS connections, in which a pair of robots must be properly aligned to form communication links;
	\item Formalization and sufficient conditions for the connectivity of the RCCS. In contrast to previous works \cite{EJC,Afonso2020}, the ensuing constraints guarantee the connectivity of a \textit{directed} network;
	\item Explicit treatment of the robot's nonlinear kinematic model using binary variables, which allows the optimization of their orientations while preserving the linear dynamics constraints required by MILP formulations.
\end{itemize}

The functionality of the method is demonstrated with a realistic robotics simulation using Gazebo and the Turtlebot3 platform. This is the first algorithm to address the motion planning of an RCCS under a VLC-based communication network.

%In this work, we propose  a MILP-based motion planning algorithm that allows the autonomous operation of the RCCS while maintaining the connectivity of the underlying communication network under VLC requirements. 
%
%Unlike previous proposals \cite{EJC},  our approach addresses the key issues of directed-LOS and connectivity of directed networks, as well as the nonlinear dynamics of nonholonomic robots.
%We formulate sufficient conditions for the connectivity of an RCCS and provide an encoding of them for MILP models as linear constraints. 
%The functionality of the method is demonstrated with a realistic robotics simulation using Gazebo and the Turtlebot3 platform. This is the first algorithm to address the motion planning of an RCCS under a VLC-based communication network.

The paper is organized as follows. Section \ref{sec:II} provides a detailed description of the RCCS-VLC motion planning problem. In Section \ref{sec:III} we propose a MILP optimization model as a solution and discuss the novel constraints required. The simulation and corresponding results are presented in Section \ref{sec:IV}. The paper is concluded in Section \ref{sec:V} where we provide final remarks and future lines of research.

We employ the following notation. Implications involving binary variables, such as $b \implies$ `$\circ$', are equivalent to $b =1 \implies$ `$\circ$'. In-degrees and out-degrees of a vertex `$v$' are denoted by $\text{deg}^-(v)$ and $\text{deg}^+(v)$, respectively. Sets are represented by calligraphic letters as in $\mathcal{M}$. Sets of integers in the interval $[n_0,n_f]$ are written as $\mathcal{I}_{n_0}^{n_f}$. A vector of size $n$ with all entries equal to one is denoted by $\mathbf{1}_n$. The Pontryagin difference operator is written as $\sim$ \cite{Baotic2009}.
% -------------------------------------------------
\section{Problem Description}\label{sec:II}

We consider an inspection problem using a wheeled RCCS similar to the one proposed in \cite{VLC_application1}. The mission must take a maximum of $\bar{N} > 0$ time steps to be completed.
The group is comprised of $n_a \in \mathbb{N}$ nonholonomic robots with  dynamics discretized with a sampling period $T > 0$ and described by the following equations, $\forall i \in \mathcal{I}_1^{n_a}$, $\forall k \in \mathcal{I}_{1}^{\bar{N}}$,
\begin{align}
	& r_{x,i}(k+1) = r_{x,i}(k) +\sigma_i(k)\cos(\psi_i(k)), \label{eq:dyn_rx} \\
	& r_{y,i}(k+1) = r_{y,i}(k) +\sigma_i(k)\sin(\psi_i(k)), \label{eq:dyn_ry}\\
	& \sigma_i(k) =  \xi_i(k)T + a_i(k)\dfrac{T^2}{2}, \label{eq:dyn_sigma}\\
	& \xi_i(k+1) = \xi_i(k) + a_i(k)T, \label{eq:lin_vel}\\
	& \psi_i(k+1) = \psi_i(k) + \Delta \psi_i(k), \label{eq:dyn_psi}
\end{align}
where  $r_{x,i} \in \mathbb{R}$ and $r_{y,i} \in \mathbb{R}$ represent the position of the $i$-th robot w.r.t. $x$ and $y$ axes of a global coordinate system, respectively; $\psi_i \in \mathbb{R}$ corresponds to the orientation of the robot w.r.t. the same coordinate system; $\sigma_i\in \mathbb{R}$, $\xi_i \in \mathbb{R}$, and $a_i \in \mathbb{R}$ are the robot's linear displacement, velocity, and acceleration, respectively. %The dynamics were discretized with a sampling period $T > 0$. %¨Fig. \ref{fig:robot_diag} provides a diagram of the robot.
%
%\begin{figure}[!ht]
%	\begin{center}	\includegraphics[width=0.3\textwidth]{./figs/pdfs/robot_diag.pdf}
	%	\end{center}
%	\caption{Robot diagram.}
%	\label{fig:robot_diag}
%\end{figure}

The components of the RCCS are divided into a static base and two types of mobile agents: the relays and a leader. The purpose of the latter is to perform the inspection of a region while the relays coordinate to keep the group connected to the base. See Fig. \ref{fig:vlc} for an example. For the sake of simplicity, the agents with indices 1 and $n_a$ are always assigned as the base and the leader, respectively. The inspection finishes after all $n_t \in \mathbb{N}$ targets, representing regions of interest in the environment, are reached by the leader. The targets are written as the polytopes $\mathcal{T}_e \subset \mathbb{R}^2,\ \forall e \in \mathcal{I}_1^{n_t}$.

The base is responsible for broadcasting high-level commands, such as trajectories to be followed and targets to be visited, or emergency stop commands.  The mobile agents operate in a region $\mathcal{A} \subset \mathbb{R}^2$ which contains $n_o \in \mathbb{N}$ obstacles $\mathcal{O}_c,\ \forall c \in \mathcal{I}_1^{n_o}$.
The space occupied by the structure of the $j$-th robot is represented by the polytope $\mathcal{R}_j(k), \forall j \in \mathcal{I}_1^{n_a}$, which is centered at $\mathbf{r}_j(k) = [r_{x,j},r_{y,j}]^\top$. The allowed operation space of robot  $i$ at time step $k$ is written as, $i \in \mathcal{I}_1^{n_a}$, 
%
%\begin{align}
%	\bar{\mathcal{A}}_i(k) = \mathcal{A} \setminus \left \{\bigcup_{c=1}^{n_o} \mathcal{O}_c \cup \bigcup_{\substack{j =1\\ j\neq i}}^{n_a}  \mathcal{R}_j(k)  \right\}.\label{eq:freeSpace}
%\end{align}
\begin{align}
	\bar{\mathcal{A}}_i(k) = \left \{ \mathcal{A} \setminus \left ( \bigcup_{c=1}^{n_o} \mathcal{O}_c \cup \bigcup_{\substack{j =1\\ j\neq i}}^{n_a}   \mathcal{R}_j(k)  \right ) \right \}  \sim \mathcal{R}_i.
	\label{eq:freeSpace}
\end{align}

\begin{figure}[!ht]
	\begin{center}	\includegraphics[width=0.45\textwidth]{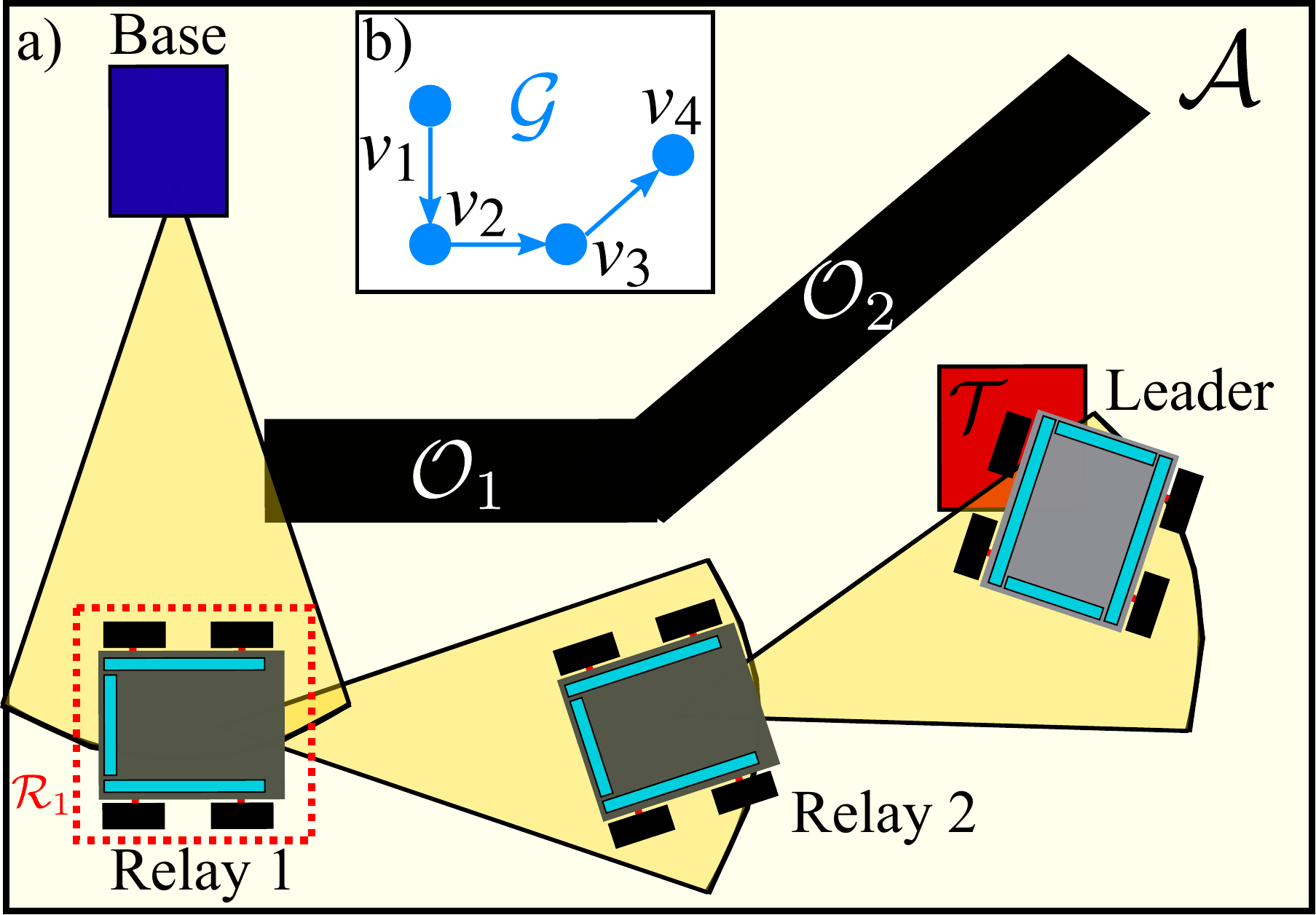}
	\end{center}
	\caption{Illustration of RCCS as devised in \cite{VLC_application1}. a) A base station with two relay robots and a leader connected under VLC communication. b) Corresponding chained communication network. }
	\label{fig:vlc}
\end{figure}

The system operates under a VLC-based communication network, as illustrated in Fig. \ref{fig:vlc}. The relay robots are equipped with solar panel receivers attached to their sides and back, whereas the leader also has a panel attached to its front. The transmitters of the relays and base are LED spotlights. The resulting communication system is simplex, i.e. the communication occurs in a single direction, and the ensuing network can be represented by the time-indexed digraph  $\mathcal{G}(k) = \{\mathcal{V},\mathcal{E}(k)\}$, where $\mathcal{V} = \{v_1,v_2,\dots,v_{n_a}\}$ is the set of vertices, with $v_1$ and $v_{n_a}$ representing the base and the leader, respectively, whereas the remaining vertices represent the relays. The set of edges is denoted by $\mathcal{E}(k)$ with edges corresponding to the time-varying connections between robots. In this context, Problem 1 specifies the issues to be addressed henceforth.

\begin{problem}
	Develop an algorithm to plan trajectories for the RCCS such that:
	\begin{enumerate}[label=\alph*)]
		\item all targets are visited by the leader within the time-limit;
		\item physical constraints of the robots are respected;
		\item the base remains connected with the leader;
		\item the robots avoid collisions;
		\item a compromise between time and control effort is minimized.

	\end{enumerate}
	\label{prob1}
\end{problem}
% -------------------------------------------------
 \section{MILP motion planning} \label{sec:III}
We propose the use of MILP to solve Problem \ref{prob1}. This technique has been shown to effectively handle motion planning of nonholonomic robots under connectivity constraints in obstacle-filled scenarios \cite{AR23}. However, linear models were employed in the motion planning module with the resulting mismatches being handled by tracking controllers. Consequently, the proposed scheme was unable to explicitly consider the robot's orientation as a decision variable.
Since the connectivity of the RCCS relies on the full pose of each robot, the nonlinear dynamics represented by (\ref{eq:dyn_rx}) and (\ref{eq:dyn_ry}) must be integrated into the model, allowing the constrained optimization of $\psi$. Additionally, the requirements of directed-LOS for connections and the connectivity maintenance of the ensuing network also represent novel challenges. The proposed solutions for these problems are developed in the following subsections.
\subsection{Dynamics constraints}
To accommodate the nonlinear dynamics of (\ref{eq:dyn_rx}) and (\ref{eq:dyn_ry}) in the MILP model, we start by assuming that the orientation is constrained to a discrete set of $n_{ori} \in \mathbb{N}$ possible values $\mathcal{S}_i = \{\hat{\psi}_{i,g}\}$, $i \in \mathcal{I}_1^{n_a}$, $\forall g \in \mathcal{I}_1^{n_{ori}}$, with 
\begin{align}
	\hat{\psi}_{i,g} = \dfrac{2\pi (g-1)}{n_{ori}}. \label{eq:disc_angles}
\end{align}

Considering a maximum horizon of $\bar{N} \in \mathbb{N}$ time steps, the dynamics are now encoded as linear constraints with binary variables, $\forall i \in \mathcal{I}_1^{n_a}$, $\forall g \in \mathcal{I}_1^{n_{ori}}$, $\forall k \in \mathcal{I}_0^{\bar{N}}$,
\begin{align}
	&r_{x,i}(k+1) \leq r_{x,i}(k)+ \nonumber \\
	& \qquad \qquad  \sigma_i(k)\cos(\hat{\psi}_{i,g}(k))+(1-b^{ori}_{i,g}(k))M,\label{const:dyn_1a}\\
	-&r_{x,i}(k+1) \leq -r_{x,i}(k)- \nonumber \\
	& \qquad \qquad  \sigma_i(k)\cos(\hat{\psi}_{i,g}(k))+(1-b^{ori}_{i,g}(k))M,\label{const:dyn_1b}\\
	&r_{y,i}(k+1) \leq r_{y,i}(k)+ \nonumber \\
	& \qquad \qquad  \sigma_i(k)\sin(\hat{\psi}_{i,g}(k))+(1-b^{ori}_{i,g}(k))M,\label{const:dyn_2a}\\
	-&r_{y,i}(k+1) \leq -r_{y,i}(k)- \nonumber \\
	& \qquad \qquad  \sigma_i(k)\sin(\hat{\psi}_{i,g}(k))+(1-b^{ori}_{i,g}(k))M,\label{const:dyn_2b}
\end{align}
where $b_{i,g}^{ori}  \in \{0,1\}$ is the orientation binary variable and $M \in \mathbb{R}$ is the ``Big-M'' constant \cite{bigm}. Through constraints (\ref{const:dyn_1a})-(\ref{const:dyn_2b}) the dynamics are written for every possible orientation in $\mathcal{S}_i$. Naturally, each robot is oriented towards one single angle at each time step and the aforementioned constraints must be imposed only for this orientation. Thus, the following constraint is imposed,  $\forall i \in \mathcal{I}_1^{n_a}$, $\forall k \in \mathcal{I}_0^{\bar{N}}$
% we would like to enforce the appropriate equations considering only one possible angle at each time step. Thus, the following constraint is proposed,  $\forall i \in \mathcal{I}_1^{n_a}$, $k \in \mathcal{I}_0^{\bar{N}}$
%
\begin{align}
	&\sum_{g=1}^{n_{ori}}b^{ori}_{i,g}(k) = 1.
	\label{const:ori_binary}
\end{align}

Constraint (\ref{const:ori_binary}) guarantees that (\ref{const:dyn_1a})-(\ref{const:dyn_2b}) are only active for one orientation at each time step, while the remaining constraints are relaxed by the ``Big-M constant''. The actual angle at time step $k$ is written as $\psi_i(k) = \boldsymbol{\Psi}^\top \mathbf{b}_{i}^{ori}(k)$ where 
\begin{align}
	& \boldsymbol{\Psi} = [\hat{\psi}_{i,1},\hat{\psi}_{i,2},\dots,\hat{\psi}_{i,n_{ori}}]^\top,\\
	& \mathbf{b}_{i}^{ori}(k) = [b^{ori}_{i,1}(k),b^{ori}_{i,2}(k),\dots,b^{ori}_{i,n_{ori}}(k)]^\top.
\end{align}

This approach allows the trigonometric functions in (\ref{const:dyn_1a})-(\ref{const:dyn_2b}) to be computed beforehand for each possible orientation, yielding linear constraints suitable for the MILP model. The remaining dynamics (\ref{eq:dyn_sigma}), (\ref{eq:lin_vel}), and (\ref{eq:dyn_psi}) are  written directly as linear constraints.

\subsection{Network connectivity}
The purpose of the communication network in an RCCS is to transmit commands from the base station to the leader. Therefore, RCCS connectivity is defined as follows.
%That is accomplished if the network allows information to flow from the base toward any robot. We start with a few results from digraph theory.
\begin{definition}\label{def:RCCS_conn}
	An RCCS is connected at time step $k$ if its communication graph $\mathcal{G}(k)$ contains a directed path between the vertices $v_1$ and $v_{n_a}$.
\end{definition}

We propose constraints for the MILP model that guarantee RCCS connectivity at each time step. The first step is to encode the vertices' in-degrees and out-degrees respectively as, $\forall i \in \mathcal{I}_1^{n_a}$, $\forall k \in \mathcal{I}_0^{\bar{N}}$,
\begin{align}
	&\text{deg}_i^-(k) = \sum_{\substack{j \in \mathcal{I}_1^{n_a}\\j \neq i}} b^{con}_{j,i}(k),\label{const:indeg} \\
	&\text{deg}_i^+(k) = \sum_{\substack{j \in \mathcal{I}_1^{n_a}\\j \neq i}} b^{con}_{i,j}(k),
	\label{const:outdeg}
\end{align}
where $b_{i,j}^{con} \in \{0,1\}$ is the connectivity binary variable. The activation of $b_{i,j}^{con}$ implies the existence of a directed edge from vertex $i$ towards vertex $j$ and the corresponding conditions for activation will be discussed in Section \ref{sec:conic_conn}. The following constraints are now proposed for RCCS connectivity, $\forall k \in \mathcal{I}_0^{\bar{N}}$, 
\begin{align}
	&\text{deg}_1^-(k) =0, \label{const:base_indeg}\\
	&\text{deg}_{n_a}^+(k) =0, \label{const:leader_outdeg}\\
	& \text{deg}_m^-(k) = 1,\ \forall m \in \mathcal{I}_2^{n_a}, \label{const:gen_indeg}\\ 
	&\text{deg}_m^+(k) = 1,\ \forall m \in \mathcal{I}_1^{n_a-1}\label{const:gen_outdeg}.
\end{align}

\begin{theorem} \label{theo:const_arbo}
	Satisfaction of constraints (\ref{const:base_indeg})-(\ref{const:gen_outdeg}) guarantees that the RCCS is connected.
\end{theorem}

\begin{proof}
	Constraints (\ref{const:base_indeg}) and (\ref{const:gen_outdeg}) enforce that $v_1$ must be connected towards some vertex $v' \in \mathcal{V} \setminus \{v_1\}$. Either, $v' = v_{n_a}$ and there is a directed path between base and leader, or $v' = v_i,\  i \in \mathcal{I}_2^{n_a-1},$ represents a relay.
	In the second case, in view of (\ref{const:gen_outdeg}), $v'$ must be connected towards some other vertex $v'' \in \mathcal{V} \setminus \{v_1,v'\}$ ($v'$ is excluded due to (\ref{const:gen_indeg})). Again, either $v'' = v_{n_a}$ and the network is connected, or $v'' \neq v_{n_a}$. In the latter case, $v''$ must connect to another a vertex $v''' \in \mathcal{V} \setminus \{v_1,v',v''\}$ in light of (\ref{const:gen_outdeg}) and so on. Since the number of vertices $n_a$ is finite, this process can be repeated at most $n_a-1$ times before reaching $v_{n_a}$. 
	Constraint (\ref{const:gen_indeg}) imposes that any subsequent relay connections that are formed due to (\ref{const:gen_outdeg}) add a distinct vertex to the chained digraph. 
	Since there is a finite number of relays, this chain is also finite and eventually must reach a sink. Constraint (\ref{const:leader_outdeg}) allows only $v_{n_a}$ to be a sink.
	Thus, the chain must contain $v_{n_a}$  and there is a directed path between the vertices representing the base and the leader. It follows that the RCCS is connected according to Definition \ref{def:RCCS_conn}.  
\end{proof}

\subsection{Directed conic LOS} \label{sec:conic_conn}

The first condition for VLC-based connections is that there must be a clear LOS between the robots. The LOS between robots $i$ and $j$ is considered clear at time step $k$ if there are no obstacles obstructing it \cite{EJC}, i.e., $i \in \mathcal{I}_1^{n_a}$, $j \in \mathcal{I}_1^{n_a} \setminus i$, $\forall c \in \mathcal{I}_1^{n_o}$,
\begin{align}
	\mathcal{L}_{i,j}(k) \cap \mathcal{O}_c = \emptyset,
	\label{const:LOS_ejc}
\end{align} 
where $\mathcal{L}_{i,j} = \{\mathbf{p} \in \mathbb{R}^2 \ \vert \ \mathbf{p} =  \mathbf{r}_i\alpha + \mathbf{r}_j(1-\alpha),\ \forall \alpha \in [0,1] \}$ with $\mathbf{r}_i = [r_{x,i}, r_{y,i}]^\top$ and $\mathbf{r}_j = [r_{x,j}, r_{y,j}]^\top$, i.e., $\mathcal{L}_{i,j}$ is the line segment connecting the position of the two robots. Condition (\ref{const:LOS_ejc}) is enforced by the following constraint, $i \in \mathcal{I}_1^{n_a}$, $j \in \mathcal{I}_1^{n_a} \setminus i$, $\forall c \in \mathcal{I}_1^{n_o}$, $\forall k \in \mathcal{I}_0^{\bar{N}}$,
\begin{align}
	b_{i,j}^{con}(k) \implies  \mathcal{L}_{i,j}(k) \cap \mathcal{O}_c = \emptyset.
	\label{const:LOS}
\end{align}
For details on the implementation of (\ref{const:LOS}), see \cite{EJC}.
\begin{lemma}\label{lemma:los}
	Constraint (\ref{const:LOS}) guarantees that if $b^{con}_{i,j}(k) = 1$, then there is a clear LOS between robots. 
\end{lemma}
\begin{proof}
	See \cite{EJC} for a proof. 
\end{proof}

The second condition is that the receiver must be within the cone of light generated by the transmitter. The LED spotlights of the robots are rigidly attached to their body, so the cone shares the same orientation as the robot. The cone is parametrized by its aperture $\theta$ and range $h$ as illustrated in Fig. \ref{fig:cone_approx}. For the sake of simplicity, its apex is positioned at the geometric center of the robot, but the method allows for different positioning depending on where the spotlights are installed. We approximate the cones for each possible orientation as polytopes with the following halfspace representation, $\forall i \in \mathcal{I}_1^{n_a}$, $\forall g \in \mathcal{I}_1^{n_{ori}}$, 
\begin{align}
	\mathcal{C}_{i,g} = \{\mathbf{z} \in \mathbb{R}^2\ \vert \ \mathbf{P}_{i,g}\mathbf{z} \leq \mathbf{q}_{i,g} \},
\end{align}
with $\mathbf{P}_{i,g} \in \mathbb{R}^{n_s \times 2 }$ and $\mathbf{q}_{i,g} \in \mathbb{R}^{n_s}$ where $n_s \in \mathbb{N}$ is the number of sides of the polygon $\mathcal{C}_{i,g} \subset \mathbb{R}^2$. 

% The cone has its apex positioned at the origin and is parameterized by the opening angle $\theta$ and height $h$ as illustrated in Fig. \ref{fig:cone_approx}. Each polytope is oriented accordingly to the angle corresponding to $g$.
%
\begin{figure}[!ht]
	\centering
	\includegraphics[width=0.30\textwidth]{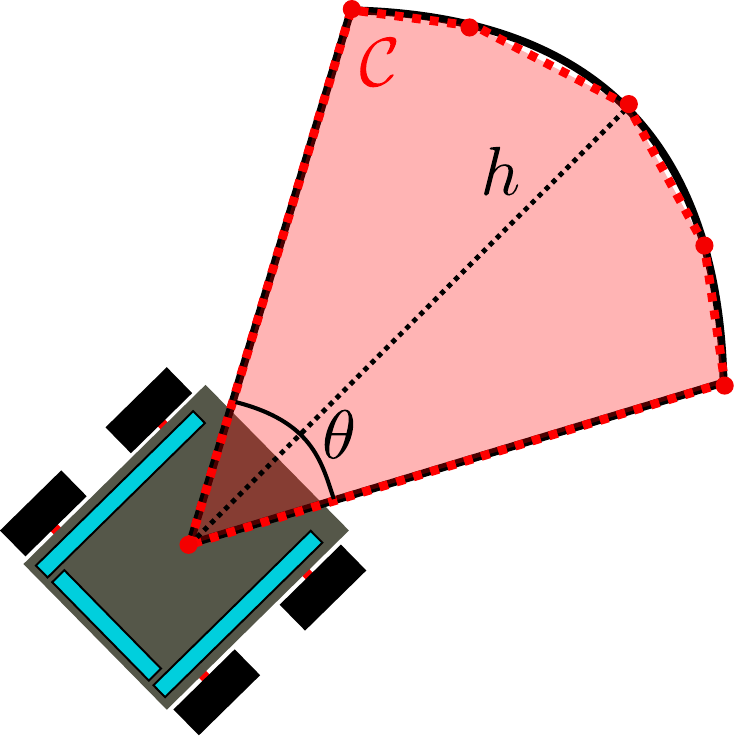}
	\caption{Original and polytopic approximation of RCCS's conic light beams.}
	\label{fig:cone_approx}
\end{figure}

Define the relative position $\tilde{\mathbf{r}}_{i,j} \triangleq \mathbf{r}_j - \mathbf{r}_i,\ i\neq j$. Then, the directed conic LOS constraint can be written as,  $\forall g \in \mathcal{I}_1^{n_{ori}}$, $\forall i \in \mathcal{I}_1^{n_a}$, $\forall j \in \mathcal{I}_1^{n_a} \setminus i$, $\forall k \in \mathcal{I}_0^{\bar{N}}$
\begin{align}
	&\mathbf{P}_{i,g} \tilde{\mathbf{r}}_{i,j}(k) \leq \mathbf{q}_{i,g}  + (1-b^{ori}_{i,g}(k))M  \nonumber \\
	& \qquad \qquad \qquad \qquad \qquad  + (1-b^{con}_{i,j}(k))M.
	\label{const:cone}
\end{align}
\begin{proposition}
	Let $\psi_{i,h}(k)$ be the orientation of the transmitter $i$ at time step $k$. Then, $b^{con}_{i,j}(k) = 1$ and the satisfaction of (\ref{const:LOS}) and (\ref{const:cone}) imply that receiver $j$ is within the unobstructed light cone of the transmitter $i$.	
	%A pair of robots $i$ and $j$ is connected under directed conic LOS requirements if  and (\ref{const:cone}) are satisfied.
\end{proposition}
\begin{proof}
	%Constraint (\ref{const:LOS}) guarantees that if $b^{con}_{i,j}(k) = 1$, then there is a clear LOS between robots. See \cite{EJC} for a proof. 
	If the orientation of transmitter $i$ is $\psi_{i,h}$ then only $b_{i,h}^{ori}=1$ due to (\ref{const:ori_binary}). Since $b^{con}_{i,j}(k) = 1$,  the particular inequality $\mathbf{P}_{i,h} \tilde{\mathbf{r}}_{i,j}(k) \leq \mathbf{q}_{i,h}$ from (\ref{const:cone})  is not relaxed by the ``Big-M'' approach and consequently must hold. Moreover, the activation of $b^{con}_{i,j}(k)$ implies a clear LOS between the robots (see Lemma \ref{lemma:los}). From these conclusions, it follows the receiver must be within an unobstructed region of the transmitter's cone of light.
\end{proof}

Finally, since the relay robots are not equipped with solar panels in their front, they are not allowed to form connections while facing each other or the base. An approximation for this condition can be attained by guaranteeing that the transmitter is never within the cone of light of the receiver. Let $\mathbf{b}^{fro}_{i,j,g} = [b^{fro}_{1,i,j,g},\dots,b^{fro}_{n_s,i,j,g}]$ be a vector of $n_s \in \mathbb{N}$ binaries variables. Then, we propose the following constraints, $\forall g \in \mathcal{I}_1^{n_{ori}}$, $\forall i \in \mathcal{I}_1^{n_a-1}$, $\forall j \in \mathcal{I}_1^{n_a-1} \setminus i$, $\forall k \in \mathcal{I}_0^{\bar{N}}$
\begin{align}
	&-\mathbf{P}_{j,g} \tilde{\mathbf{r}}_{j,i}(k) \leq -\mathbf{q}_{j,g} + (\mathbf{1}_{n_s} - \mathbf{b}^{fro}_{i,j,g}(k)) M, \label{const:front_1}\\
	& b^{con}_{i,j}(k) \leq \sum_{z=1}^{n_s} b^{fro}_{z,i,j,g}(k) + (1-b^{ori}_{j,g}(k)). \label{const:front_2}
\end{align}

Constraint (\ref{const:front_1}) is similar to the obstacle avoidance constraints in \cite{Afonso2020}. It guarantees that the $i$-th robot is not within the cone of light of the $j$-th robot if at least one entry of $\mathbf{b}^{fro}_{i,j,h}$ takes the value of 1.
Constraint (\ref{const:front_2}) conditions the activation of $b^{con}_{i,j}$ to at least one entry $\mathbf{b}^{fro}_{i,j,g}$ taking the value of 1. The orientation binary guarantees that this constraint is only enforced for the cone with appropriate orientation.

\subsection{Rotation dynamics compensation}
In the proposed motion planning algorithm the assumption that the robots can reach the commanded orientations instantly is made, as shown in (\ref{eq:dyn_psi}). To adhere to this assumption, we compensate for the rotation dynamics delay through the manipulation of the cone polytope.

Let $\Delta t$ denote the maximum time, in seconds, taken by a robot to rotate a maximum orientation increment of $\Delta \psi^{max}_i(k)$, with the rotation increment being defined as $\Delta \psi_{i}(k)  \triangleq \psi_{i,g}(k+1)-\psi_{i,g}(k)$. To compensate for the delay in rotation, one must ensure that the receiver robot is still within the light cone of the transmitter after moving for $\Delta t$ seconds. The maximum displacement of a robot during this time is $d = \xi^{max} \Delta t$, where $\xi^{max} >0$ is the maximum linear velocity of the robots. Thus, we compensate for the potential movement of the receiver by planning the trajectories considering light cone polytopes $\mathcal{C}_{i,g}$ shrinked by $d$ in all directions. This is accomplished with the following operation, $\forall i \in \mathcal{I}_1^{n_a}$, $\forall g \in \mathcal{I}_1^{n_{ori}}$,
\begin{align}
	\bar{\mathcal{C}}_{i,g} = 	\mathcal{C}_{i,g} \sim \mathcal{B}_d,
\end{align}
where $\bar{\mathcal{C}}_{i,g} \subset \mathbb{R}^2$ is a shrinked cone and $\mathcal{B}_d \subset \mathbb{R}^2$ a ball of radius $d$ and centered at the origin. 

To enforce the orientation increment to be bounded, we start by constraining this variable to the symmetric interval $[-\pi, \pi]$. Notice that since the orientations can take values between $[0,2\pi]$, as determined in (\ref{eq:disc_angles}), the increments must be corrected to lie within the aforementioned symmetric interval for proper imposition of the bounds. 
The correction is accomplished by introducing the following constraints to the optimization problem, $\forall i \in \mathcal{I}_1^{n_a}$, $\forall k \in \mathcal{I}_0^{\bar{N}}$,
%This is required for proper computation of the increments, since the orientations can take values in the interval between $[0,2\pi]$, as determined in (\ref{eq:disc_angles}).
 %The correction is performed by the following constraints added to the optimization problem

%The increment is the difference between the angles in two subsequent time steps, as such, proper increment values are computed using the symetric interval instead of a
% discontinuities in the trigonometric functions. This is accomplished with the following constraints, $\forall i \in \mathcal{I}_1^{n_a}$, $\forall k \in \mathcal{I}_0^{\bar{N}}$, 
%
\begin{align}
	&\Delta \psi_i(k) \leq \pi +  b_{i,1}^{disc}(k)M, \label{const:disc_1} \\
	-&\Delta \psi_i(k) \leq \pi +  b_{i,2}^{disc}(k)M,\label{const:disc_2} \\
	& \Delta \bar{\psi_i}(k) = \Delta \psi_i(k) - 2\pi(1-b_{i,1}^{disc}(k))+  \nonumber \\
	& \qquad \qquad \qquad \qquad \quad  2\pi(1-b_{i,2}^{disc}(k)). \label{const:disc_3} 
\end{align}

Constraints (\ref{const:disc_1}) and (\ref{const:disc_2}) enforce that if $\Delta \psi_i$ lie outside the desired interval of $[-\pi, \pi]$, then one of the binary variables $b^{disc}_{i,1} \in \{0,1\}$ or $b^{disc}_{i,2} \in \{0,1\}$ are activated, whereas (\ref{const:disc_3}) updates a new variable $\Delta \bar{\psi}_i(k)$ with the appropriate correction by summing or subtracting $2\pi$. Now, the orientation increment bounds are enforced to be within the interval $[\Delta \psi^{min},\Delta \psi^{max}]$ by the constraints, $\forall i \in \mathcal{I}_1^{n_a}$, $\forall k \in \mathcal{I}_0^{\bar{N}}$, 
\begin{align}
	&\Delta \bar{\psi}_i(k)  \leq \Delta \psi^{max},\label{const:ori_increment1}\\
	-&\Delta \bar{\psi}_i(k)  \leq \Delta \psi^{min}.\label{const:ori_increment2}
\end{align}
%
%whereas the discontinuities in the trigonometric functions are handled by keeping $\Delta \psi_i$ within the interval $[-\pi,\pi]$
%\begin{align}
%&\Delta \psi_i(k) \leq \pi +  b_{i,1}^{disc}(k)M, \label{const:disc_1} \\
%%
%  -&\Delta \psi_i(k) \leq \pi +  b_{i,2}^{disc}(k)M,\label{const:disc_2} \\
%%
%& \Delta \psi_i(k) = \Delta \psi_i(k) - 2\pi(1-b_{i,1}^{disc}(k)´+  \nonumber \\
%& \qquad \qquad \qquad \qquad \quad  2\pi(1-b_{i,2}^{disc}(k)). \label{const:disc_3} 
%\end{align}
%Constraints (\ref{const:disc_1}) and (\ref{const:disc_2}) enforce that if $\Delta \psi_i$ lie outside the desired interval, then one of the binary variables $b^{disc}_{i,1} \in \{0,1\}$ or $b^{disc}_{i,1},2 \in \{0,1\}$ are activated, whereas (\ref{const:disc_3}) updates $\Delta \psi_i(k)$ with the appropriate correction by summing or subtracting $2\pi$.
\subsection{Optimization problem}
Let $\boldsymbol{\lambda} = [\mathbf{r}_i, \xi_i, a_i, N, b^{tar}_e, b^{con}_{i,j},b^{ori}_{i,j},b^{disc}_{i,1},b^{disc}_{i,2} ]$ denote the vector of all optimization variables. Then the motion planning optimization is written as the following. 
%\vspace{-20pt}
\begin{subequations}
	\allowdisplaybreaks
	\begin{flalign}
		%	&~\nonumber\\ %[10pt] 
		&\text{\textbf{Motion planning problem.}} \nonumber \\
		&\underset{\boldsymbol{\lambda}}{\textrm{min}}\  N +   \sum_{k=0}^{N} \sum_{i=1}^{n_a} \left ({\gamma_a}\Vert a_{i}(k)  \Vert_1 +  \gamma_{\Delta \psi}\Vert \Delta \psi_{i}(k) \Vert_1 \right)
		\label{eq:cost} \\
		&\textrm{s.t., } \forall i \in \mathcal{I}_1^{n_a}, \  \forall c \in \mathcal{I}_1^{n_o},\ \forall e \in \mathcal{I}_1^{n_t},\ \forall k \in \mathcal{I}_0^{N}, \nonumber \\
		% Dynamics
		&\text{(\ref{eq:dyn_sigma}), (\ref{eq:lin_vel}), (\ref{const:dyn_1a}-\ref{const:ori_binary}), (\ref{const:indeg}-\ref{const:gen_outdeg}), (\ref{const:LOS}), (\ref{const:cone}-\ref{const:front_2}) (\ref{const:disc_1}-\ref{const:ori_increment2})}  \nonumber\\
		% State bounds
		& \xi_i^{min} \leq \xi_i(k) \leq \xi_i^{max}, \label{const:lin_vel_bound}\\
		&  a_i^{min} \leq a_i(k) \leq a_i^{max}, \label{const:acc_bound} \\
		& \mathbf{r}_{i}(k+1) \in \bar{\mathcal{A}}_i(k+1), \label{const:freeSpace}\\
		& b^{tar}_e(k) \implies \mathbf{r}_{n_a}(k) \in  \mathcal{T}_e, \label{const:target_1} \\
		& \sum_{k=0}^{N}b^{tar}_e(k) = 1 \label{const:target_2}.
	\end{flalign}
\end{subequations}

The cost (\ref{eq:cost}) is a compromise between time expenditure, represented by the variable horizon $N \in \mathbb{N}$ \cite{varHor}, the acceleration effort weighted by $\gamma_a >0$, and the orientation increments weighted by $\gamma_{\Delta \psi} >0$. Constraints (\ref{const:lin_vel_bound}) and (\ref{const:acc_bound}) enforce the linear velocities and accelerations to be within the bounds determined by the intervals $[\xi^{min}$, $\xi^{max}]$ and $[a^{min}$, $a^{max}]$, respectively. Collision avoidance is guaranteed by constraint (\ref{const:freeSpace}). 
Finally, (\ref{const:target_1}) and (\ref{const:target_2}) enforce that at the time step corresponding to the optimal horizon $N \leq \bar{N}$, all targets must have been visited by the leader. 
%Finally, (\ref{const:target_1}) and (\ref{const:target_2}) enforce that all targets are visited by the leader within the maximum horizon $\bar{N}$, which is an upper bound to the variable horizon $N$. 

The implementation of (\ref{eq:cost})-(\ref{const:target_1}) is omitted for the sake of brevity. 
The reader is referred to \cite{EJC} and \cite{Afonso2020} for a thorough presentation. We also provide the code for the motion planning algorithm as open source\footnote{\url{https://gitlab.com/caregnato_neto_open/MILP_VLC_motion_planning}}.

%%%%%%%
\section{Results} \label{sec:IV}
We evaluate the proposed formulation through a simulation of the Turtlebot3 robot platform in the Gazebo environment. The orientation control law was designed as a proportional controller with gain $K_p = 6.5$, whereas the linear velocity uses only a feedforward controller. After experiments with the proposed control laws considering the robot's maximum linear velocity to be $\xi^{max} = 0.22$ m/s, we concluded that in the worst case scenario $\Delta t \approx 1$ s considering orientation increment bounds of $\Delta \psi^{min} = -60^\circ$ and $\Delta \psi^{max} = 60^\circ$ resulting in $d \approx 0.22$ m. Safety margins of $\epsilon = 0.08$ m were added to all polytopes to handle potential disturbances and tracking errors.  

The bodies of the robots were approximated by the square that circumscribes the circle with a diameter equal to the length of the robot, $0.178$ m. The parameters of the light cone were selected as $\theta = 60^\circ$, $h = 2.8$ m. The base is introduced in the problem using the same constraints but considering a static agent. Since the proposed algorithm operates only as an open-loop trajectory planner and the Turtlebot3 operates with relatively low velocities, we employed a sampling period of $T = 4$ s
 The maximum horizon was selected as  $\bar{N} = 8$ , corresponding to a maneuver with a maximum duration of 32 s. The total number of orientations was chosen as $n_\psi = 12$, corresponding to steps of $30^\circ$. The light cone was approximated by a polytope with $n_s = 6$ sides. The optimization weights were selected as $\gamma_a = 0.1$ and $\gamma_{\Delta \psi} = 0.1$.

The planning problem was solved in a notebook equipped with an Intel$^\text{\textregistered}$ i5-1135G7 (2.40 GHz clock) CPU and 16 GB of RAM using the Gurobi solver \cite{gurobi}. The optimization model was built using the Yalmip \cite{Lofberg2004} and MPT toolboxes \cite{mpt}. We established a limit computation time of $360$ s resulting in a feasible solution with an optimality gap of $4.2\%$. 

Fig. \ref{fig:traj}a shows the commanded and followed trajectories of the RCCS in an environment with width and length of $4$ m and $3$ m, respectively, as well as two targets and obstacles.  The group was able to track the trajectories, demonstrating that the algorithm was able to compute adequate motion considering the physical limitations of the robots. The maximum tracking error was $0.078$ m, which is within the expected margins. The optimal horizon 
was $N=7$, corresponding to the mission being finished after 28 seconds; after this period, the group is commanded to stop translation. Fig 3b depicts the terminal pose of the RCCS.

\begin{figure}
	\centering
	\includegraphics[width=0.4\textwidth]{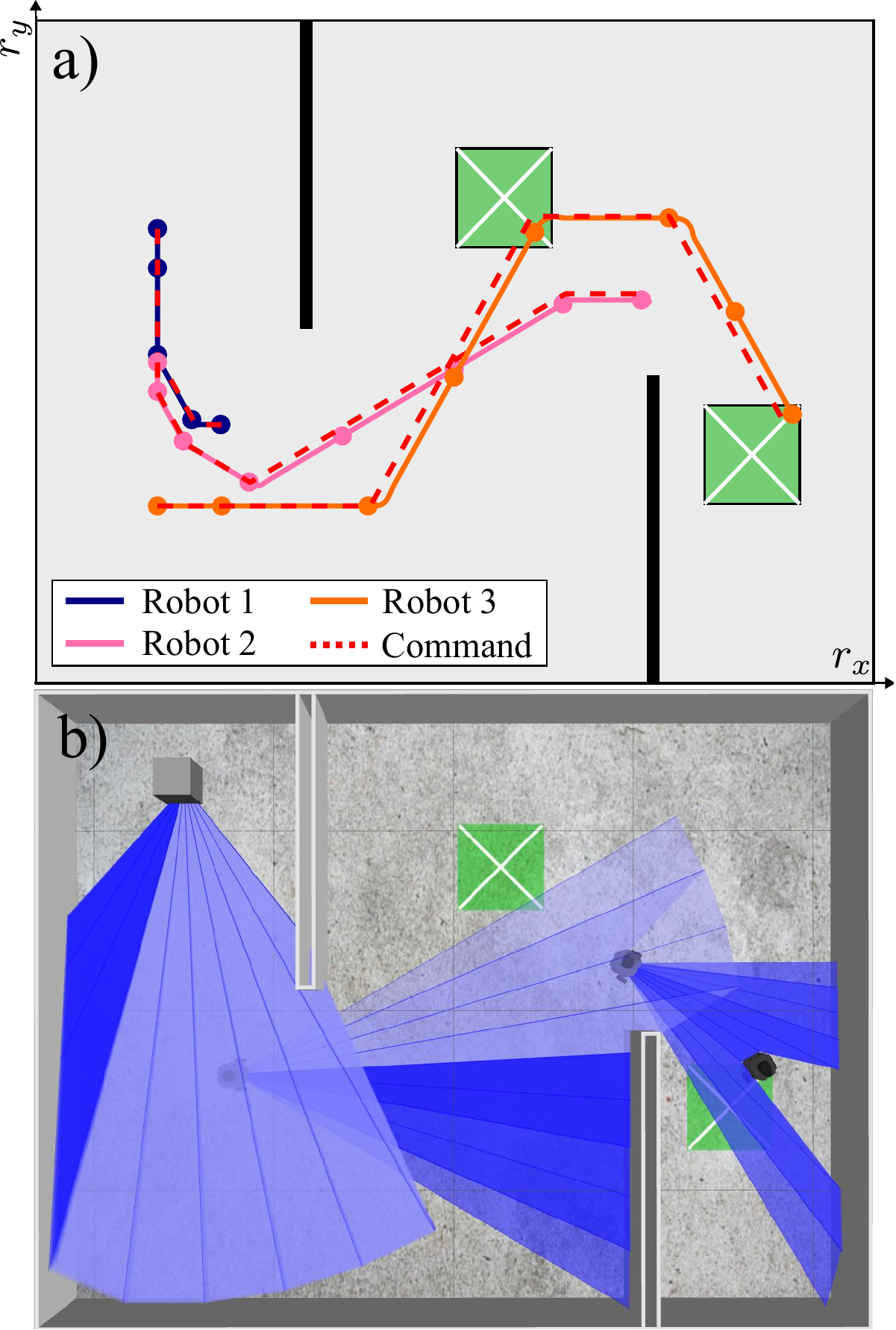}
	\caption{a) Trajectories computed by the MILP algorithm and tracking performance of the Turtlebots3 during the Gazebo simulation depicted in b) Final pose of the robots in the simulator. Green squares with white \textit{x} markers are the targets. Simulation video available at: \protect \url{https://www.youtube.com/watch?v=xDju9YRVrtg} }
	\label{fig:traj}
\end{figure}

\begin{figure*}[!ht]
	\centering
	\includegraphics[width=7in]{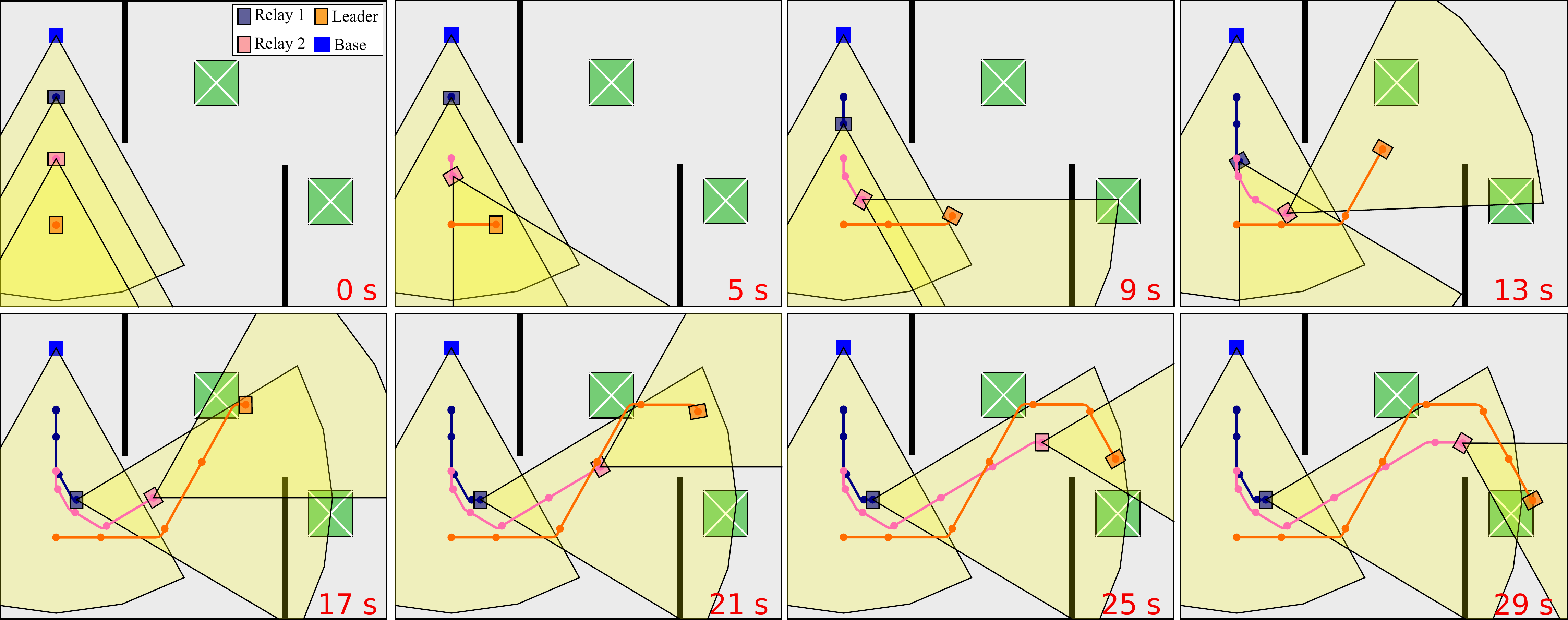}%
	\caption{Snapshots of the Turtlebots' position during Gazebo simulation at critical time steps. Green and black polygons are targets and obstacles, respectively.}
	\label{fig:matlab_snapshots}
\end{figure*}
\begin{figure*}
	\centering
	\includegraphics[width=7in]{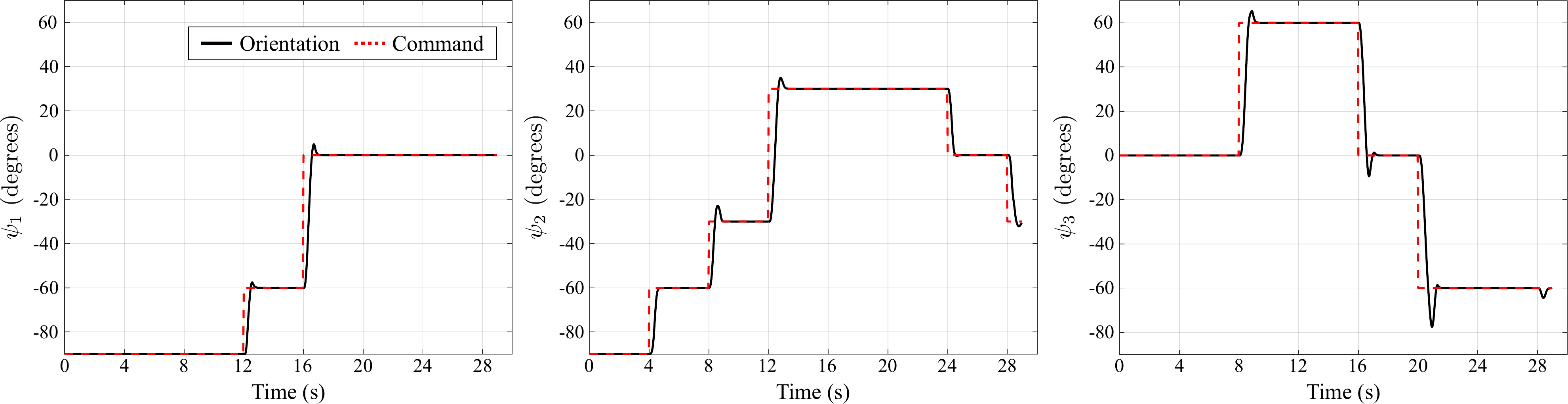}%
	\caption{Orientation of each robot during the maneuver and corresponding commands.}
	\label{fig:results_orient}
\end{figure*}

The snapshots of the maneuver at each instant $kT+\Delta t$ and the corresponding light beams of the base and robots are presented in Fig. \ref{fig:matlab_snapshots}. Using the planned trajectories the relays were able to successfully coordinate, maintaining the connectivity of the network at the appropriated time steps while the leader reached both targets. No collisions occurred. The commanded orientation increments were always kept within the desired bounds, as illustrated by Fig. \ref{fig:results_orient}, allowing each robot to reach the desired orientation within the required time interval.
%

%
% \begin{figure}[!ht]
	%	\centering
	%	\includegraphics[width=0.45\textwidth]{./figs/pdfs/trajectories.pdf}
	%	\caption{Trajectories computed by the MILP algorithm and tracking performance of the Turtlebots3 in a Gazebo simulation.}
	%	\label{fig:traj}
	%\end{figure}
	%%	
	% \begin{figure}[!ht]
		%	\centering
		%	\includegraphics[width=0.45\textwidth]{./figs/pdfs/gazebo_sim.pdf}
		%	\caption{Trajectories computed by the MILP algorithm and tracking performance of the Turtlebots3 in a Gazebo simulation.}
		%	\label{fig:gazebo_sim}
		%\end{figure}
		%
%
%
\section{Conclusion}\label{sec:V}
This work addressed the motion planning problem of an RCCS under VLC connectivity constraints. The proposed MILP-based formulation was able to compute a high-quality solution in terms of trajectories for the group of robots, resulting in appropriate coordination. The functionality was demonstrated using realistic Gazebo simulations where a group of Turtlebots completed a mission while maintaining connectivity and avoiding collisions. Future formulations can integrate the decision problem of switching the size of the luminosity cones to preserve battery life. Alternative connectivity conditions can also be studied to allow a flexible network topology. Experiments with real robots would further validate the proposal.

\bibliographystyle{IEEEtran}
\bibliography{IEEEexample}

\end{document}